\documentclass[11pt]{article}

\usepackage[T1]{fontenc}
\usepackage[utf8]{inputenc}
\usepackage{lmodern}
\usepackage[margin=1in]{geometry}
\usepackage{microtype}
\usepackage{amsmath,amssymb,amsthm,mathtools}
\usepackage{booktabs}
\usepackage{array}
\usepackage{tabularx}
\usepackage{makecell}
\usepackage{threeparttable}
\usepackage{siunitx}
\usepackage{graphicx}
\usepackage{tikz}
\usetikzlibrary{arrows.meta,positioning,fit}
\usepackage{float}
\usepackage{enumitem}
\usepackage[numbers,sort&compress]{natbib}
\usepackage{xcolor}
\usepackage{hyperref}

\hypersetup{
  colorlinks=true,
  linkcolor=blue!55!black,
  citecolor=blue!55!black,
  urlcolor=blue!55!black
}

\newtheorem{theorem}{Theorem}[section]
\newtheorem{corollary}[theorem]{Corollary}

\newtheorem{remark}[theorem]{Remark}
\theoremstyle{definition}

\newcommand{\E}{\mathbb{E}}

\newcommand{\norm}[1]{\left\lVert #1 \right\rVert}
\newcommand{\method}{CC-AOS}
\newcommand{\cost}{\lambda}
\newcommand{\Htrain}{\mathcal{H}_{\mathrm{train}}}
\newcommand{\Costs}{\Lambda}

\title{\textbf{CC-AOS: Cost- and Horizon-Conditioned Amortized Backward Induction for Finite-Horizon Optimal Stopping}}
\author{Tianwei Yu}
\date{Draft v1 -- July 23, 2026}

\begin{document}
\maketitle

\begin{abstract}
Finite-horizon optimal stopping is a central problem in early time-series classification, where a system must decide at each sequence prefix whether the expected benefit of another observation justifies its acquisition cost. Existing data-driven backward-induction methods typically solve each cost--horizon operating point separately, so changing operating conditions requires repeated optimization and separate model stacks, making continuous cost adaptation and multi-horizon deployment inefficient. We propose \method{} (Cost- and Horizon-Conditioned Amortized Optimal Stopping), a structured amortized solver for a family of finite-horizon stopping problems with continuous costs and multiple horizons. \method{} learns a shared continuation-value model conditioned on the current state, absolute time, remaining horizon, and acquisition cost through joint amortized fitted backward induction. We establish that the exact value and continuation functions are nondecreasing, concave, and horizon-dependently Lipschitz in cost, encode these properties in the model architecture, and derive residual-based bounds on value and policy errors. Experiments on controlled Gaussian and time-varying non-Gaussian processes and the FordA engine-noise time-series benchmark compare \method{} with representative per-operating-point backward-induction solvers and tuned static stopping rules. At six unseen FordA cost--horizon pairs, one \method{} checkpoint achieved a lower terminal-risk-plus-sampling-cost objective than independently fitted Convex Function Learning at all six pairs, with an average reduction of 15.75\%, while matching the tuned static thresholds on average.
\end{abstract}

\noindent\textbf{Keywords:} optimal stopping; early time-series classification; backward induction; amortized optimization; shape-constrained neural networks.

\section{Introduction}

\subsection{Finite-horizon stopping under multiple operating points}

Sequential decision systems often need to act before all potentially available information has been collected. In early time-series classification, for example, a prediction can be made from the current prefix or postponed until another observation is acquired. Acting early saves time, energy, computation, or measurement effort, but may increase the terminal decision error; waiting can improve the prediction, but its benefit must justify both the additional acquisition cost and the delay \citep{wald1947sequential,xing2009early}. This tension makes the stopping rule, rather than classification accuracy alone, central to the decision problem.

The difficulty becomes sharper under a finite observation budget. Near the beginning of a sequence, several future observations remain available, whereas near the deadline the same current evidence may warrant a different action. Backward induction captures this stage dependence by comparing the terminal risk \(g(s)\) of stopping at state \(s\) with the complete continuation risk \(\cost+C_h(s;\cost)\). Here, \(\cost\) is the cost of the next observation and \(C_h\) is the expected future optimal risk after that observation, with \(h\) acquisitions remaining. Thus, the continuation value and the resulting stopping boundary depend jointly on the current state, the remaining horizon, and the sampling cost \citep{chow1971great,peskir2006optimal}.

In many applications, however, neither the cost nor the available horizon is fixed once and for all. Resource availability may change, several latency budgets may need to be supported, or a deployed system may need to trade accuracy against delay at different operating conditions. The relevant object is therefore not a single stopping policy, but a related family of finite-horizon problems indexed by the sampling cost and horizon. Accordingly, this work studies a unified model that learns this family coherently and can be queried at cost--horizon pairs that were not solved separately during training.

\subsection{Limitations of per-operating-point solvers}

Backward induction provides a principled solution once an operating point is specified. In data-driven settings, its conditional expectations can be estimated by regression, yielding stagewise approximations of the continuation value \citep{longstaff2001valuing,tsitsiklis2001regression}. Recent early-classification methods, including FIRMBOUND, combine learned sequential statistics with this finite-horizon recursion \citep{ebihara2025firmbound}. Despite differences in their estimators, such solvers generally treat the sampling cost and maximum horizon as fixed problem parameters: a continuation model is fitted for the requested operating point.

This formulation is effective when only one setting is needed, but it scales poorly to a changing collection of settings. Each new cost or horizon can require new recursive targets, another optimization run, and another stack of stagewise models. More importantly, a collection of independently fitted solvers gives values only at the operating points that were explicitly solved. It neither defines how the continuation value should behave between those points nor ensures that their approximations form one Bellman-consistent family.

Parameter conditioning appears to offer an immediate remedy, but an unconstrained shared network leaves an important part of the stopping problem unused. For a fixed state and remaining horizon, the exact optimal value has a particular geometry as sampling cost varies: it is nondecreasing, concave, and Lipschitz with a constant determined by the number of acquisitions still available. A useful amortized solver should therefore do more than compress several separately trained models. It should couple the operating points through a shared Bellman learning procedure while preserving the cost structure of the underlying finite-horizon problem.

\subsection{Structured amortized optimal stopping}

These requirements motivate Cost- and Horizon-Conditioned Amortized Optimal Stopping (\method{}). The starting point is that stopping problems at different costs and horizons are not unrelated tasks: they share the same state process, terminal risk, and finite-horizon recursion. \method{} exposes this common structure by treating cost and remaining horizon as conditioning variables of a single continuation model \(C_\theta(s,t,h,\cost)\), rather than as fixed settings chosen before optimization.

The model is learned by joint amortized fitted backward induction. Recursive targets are constructed over continuously sampled costs and transitions drawn from multiple training horizons, so all operating points contribute to one conditional value surface. To make the learned surface reliable along the continuous cost axis, a normalized soft-min affine head encodes the monotonicity, concavity, and remaining-horizon-dependent Lipschitz structure established for the exact problem.

The resulting formulation connects representation, learning, and analysis around the same continuation-value family. We derive residual-based bounds that relate Bellman approximation error to value-function error and the performance loss of the induced stopping policy, and evaluate whether one checkpoint transfers to unseen cost--horizon pairs. Experiments span controlled Gaussian and time-varying non-Gaussian processes as well as the FordA engine-noise benchmark under causal prefix-only normalization. On FordA, \method{} improves over independently fitted Convex Function Learning at all six unseen operating points, with five paired confidence intervals entirely below zero, while remaining essentially tied on average with a strong validation-tuned static threshold.

\subsection{Contributions}

The main contribution of this paper is \method{}, a structured amortized solver for a family of finite-horizon stopping problems with continuous sampling costs and multiple horizons. The solver is built on three connected contributions:
\begin{enumerate}[leftmargin=*,label=(\arabic*)]
  \item \textbf{Shared cost--horizon continuation model.} \method{} represents the stopping-problem family with a single continuation model \(C_\theta(s,t,h,\cost)\), conditioned on the current state, absolute time, remaining horizon, and sampling cost. One trained checkpoint can therefore be queried directly to construct stopping policies at multiple and held-out cost--horizon pairs.
  \item \textbf{Joint amortized training method.} We develop a joint amortized fitted backward-induction method that constructs recursive Bellman targets over continuously sampled costs and multiple training horizons. The method learns the shared continuation family within one optimization process while retaining the recursive coupling across decision stages and operating conditions.
  \item \textbf{Structure-preserving architecture and theoretical analysis.} We establish the monotonicity, concavity, and remaining-horizon-dependent Lipschitz properties of the exact value and continuation functions with respect to sampling cost. We encode these properties in a normalized soft-min affine continuation head and derive residual-based bounds for value-function error and stopping-policy suboptimality.
\end{enumerate}

\section{Related Work}

\method{} lies at the intersection of three connected lines of work. Finite-horizon optimal stopping and early time-series classification define the underlying risk--acquisition trade-off. Data-driven continuation learning and sequential density-ratio estimation provide the basis for learning finite-horizon stopping decisions from trajectories. Parameter-conditioned learning and shape-constrained approximation then provide the tools needed to share a model across operating conditions without discarding known value-function structure.

\subsection{Finite-horizon optimal stopping and early time-series classification}

Classical sequential analysis formalizes how evidence should be accumulated before a decision is made \citep{wald1947sequential}. For simple hypothesis tests, the sequential probability ratio test gives a canonical stopping rule with established optimality properties \citep{wald1948optimum}. When observations are limited by a finite deadline, however, the decision becomes stage dependent: the value of another observation changes with the time remaining. Dynamic programming and backward induction provide the standard formulation for this finite-horizon setting \citep{chow1971great,peskir2006optimal}.

Early time-series classification places the same stopping problem in a predictive setting. The foundational formulation seeks accurate decisions from short sequence prefixes \citep{xing2009early,xing2012early}. Non-myopic approaches explicitly account for the possible benefit of future observations \citep{dachraoui2015nonmyopic}, while later work jointly optimizes accuracy and earliness or learns adaptive halting policies \citep{mori2018simultaneously,hartvigsen2019adaptive}. Other methods emphasize reliable early decisions through ensemble agreement or explicit control of the accumulated accuracy gap \citep{schaefer2020teaser,ringel2024accumulated}. \method{} follows this cost-sensitive finite-horizon view: stopping incurs a terminal prediction risk, whereas continuing incurs an acquisition cost and preserves the option to decide later.

\subsection{Data-driven backward induction and amortized value learning}

Finite-horizon backward induction requires conditional expectations of future optimal values. Least-squares Monte Carlo and regression-based dynamic programming estimate these continuation values stage by stage from simulated trajectories \citep{longstaff2001valuing,tsitsiklis2001regression}. Statistical-learning analyses further connect continuation regression to optimal-stopping approximation \citep{egloff2005montecarlo}. Neural approaches extend data-driven stopping to high-dimensional state spaces \citep{becker2019deep,becker2021highdimensional}, while interpretable policy classes and randomized neural features provide alternative approximations \citep{ciocan2022interpretable,herrera2024randomized}.

For early classification, sequential density-ratio estimation learns class evidence directly from sequence prefixes and optimizes the speed--accuracy trade-off \citep{ebihara2021sequential,miyagawa2021power}. Recent work further combines data-derived sequential statistics, posterior terminal risk, and fitted finite-horizon backward induction \citep{ebihara2025firmbound}. Within this data-driven stopping framework, \method{} changes the granularity of the solver from a continuation stack for one prescribed cost and horizon to a continuation family indexed jointly by both quantities.

This extension draws on the broader principle of amortized optimization, which learns a parameter-to-solution mapping so that related problem instances share computation \citep{amos2023tutorial}. Universal value-function approximators similarly condition a value model on task variables \citep{schaul2015universal}, and multi-objective reinforcement learning conditions policies on changing preference weights \citep{abels2019dynamic,yang2019generalized}. These works provide the conceptual basis for treating cost and horizon as model inputs. \method{} specializes that principle to finite-horizon continuation learning by constructing recursive targets jointly over continuous costs and multiple training horizons within one fitted backward-induction method.

\subsection{Shape-constrained approximation for continuation values}

Shape-constrained neural networks provide mechanisms for incorporating known functional structure into a learned approximation. Monotonic networks enforce prescribed coordinatewise responses \citep{sill1997monotonic}, while input-convex neural networks encode convexity with respect to selected inputs \citep{amos2017inputconvex}. Subsequent architectures broaden the design space for monotonic approximation and improve its flexibility \citep{wehenkel2019unconstrained,runje2023constrained}.

\method{} builds on these architectural ideas using constraints derived from the finite-horizon Bellman problem itself. With \(h\) acquisitions remaining, the exact value is nondecreasing, concave, and \(h\)-Lipschitz in sampling cost; the corresponding continuation value is \((h-1)\)-Lipschitz. The normalized soft-min affine head imposes this horizon-dependent cost geometry while allowing state, time, and remaining horizon to determine the affine components. In this way, data-driven backward induction, amortized parameter conditioning, and shape-constrained approximation are combined within a single continuation-value family.

\section{Problem Formulation}

\subsection{Sequential information and stopping risk}

Let \(Y\in\{1,\ldots,K\}\) be the class label and let \(X_{1:t}=(X_1,\ldots,X_t)\) be the sequence prefix observed at time \(t\). The available information is represented by the filtration \(\mathcal{F}_t=\sigma(X_{1:t})\). A causal state estimator maps each prefix to the decision state
\begin{equation}
S_t=\phi(X_{1:t})\in\mathcal{S},\qquad t=1,\ldots,H.
\label{eq:decision-state}
\end{equation}
We assume that \(S_t\) is a Markov state for the stopping problem. For a time-inhomogeneous process, absolute time is included in an augmented state; the learned model in Section~4 retains \(t\) explicitly. When \(S_t\) is learned, it is treated as the working decision state rather than claimed to be an exact sufficient statistic.

Let \(\mathcal{A}\) be the terminal action space and let \(\ell(a,Y)\) be the loss incurred by action \(a\). The risk of stopping at state \(s\) and issuing the Bayes action is
\begin{equation}
g(s)
=
\min_{a\in\mathcal{A}}
\E\!\left[\ell(a,Y)\mid S_t=s\right].
\label{eq:terminal-risk}
\end{equation}
For zero--one classification loss, this reduces to \(g(s)=1-\max_y\Pr(Y=y\mid S_t=s)\). We assume \(0\le g(s)\le B<\infty\). Acquiring one additional observation incurs a constant cost \(\cost\ge0\).

\subsection{Finite-horizon objective and Bellman recursion}

The maximum sequence horizon is \(H\). At time \(t\), the number of additional observations still available is
\begin{equation}
h=H-t.
\label{eq:remaining-horizon}
\end{equation}
Let \(\mathfrak{T}_{t,h}\) be the set of \(\{\mathcal{F}_{t+k}\}_{k=0}^h\)-adapted stopping times that count additional acquisitions, so every \(\tau\in\mathfrak{T}_{t,h}\) takes values in \(\{0,\ldots,h\}\). Conditional on \(S_t=s\), the finite-horizon optimal value is
\begin{equation}
V_h(s;\cost)
=
\inf_{\tau\in\mathfrak{T}_{t,h}}
\E\!\left[
g(S_{t+\tau})+\cost\tau
\mid S_t=s
\right].
\label{eq:stopping-objective}
\end{equation}
Thus, \(\tau=0\) represents immediate stopping, while \(\tau=k\) incurs the cost of \(k\) additional observations. We assume that \(g(S_{t+\tau})\) is integrable and that the state transition law is independent of \(\cost\).

For any bounded measurable function \(f\), define the one-step transition operator
\begin{equation}
(Pf)(s)
=
\E\!\left[f(S_{t+1})\mid S_t=s\right].
\label{eq:transition-operator}
\end{equation}
The post-acquisition continuation value for \(h\ge1\) is then
\begin{equation}
C_h(s;\cost)
=
(PV_{h-1})(s;\cost)
=
\E\!\left[V_{h-1}(S_{t+1};\cost)\mid S_t=s\right].
\label{eq:exact-continuation}
\end{equation}
The current acquisition cost is excluded from \(C_h\); the complete continuation risk is \(\cost+C_h\). Backward induction gives
\begin{align}
V_0(s;\cost)
&=g(s),
\label{eq:terminal-boundary}\\
V_h(s;\cost)
&=
\min\!\left\{
g(s),\,\cost+C_h(s;\cost)
\right\},
\qquad h\ge1.
\label{eq:bellman}
\end{align}
Accordingly, the Bayes-optimal stopping rule is
\begin{equation}
\pi_h^\star(s;\cost)
=
\begin{cases}
\mathrm{stop}, & g(s)\le \cost+C_h(s;\cost),\\
\mathrm{continue}, & g(s)>\cost+C_h(s;\cost),
\end{cases}
\qquad h\ge1,
\label{eq:optimal-policy}
\end{equation}
with mandatory stopping at \(h=0\).

\subsection{The cost--horizon continuation family}

An operating point is \(\omega=(\cost,H)\), where \(\cost\in\Costs\) and \(H\) belongs to a set of admissible finite horizons. For a fixed \(\omega\), conventional backward induction produces the stack
\[
\bigl\{C_h(\,\cdot\,;\cost)\bigr\}_{h=1}^{H-1}.
\]
Across operating points, these stacks solve the same recursion under different cost and horizon conditions. Our approximation target is therefore one conditional function
\begin{equation}
C_\theta:
\mathcal{S}\times\mathbb{N}\times\mathbb{N}\times\Costs
\longrightarrow [0,B],
\qquad
C_\theta(s,t,h,\cost)\approx C_h(s;\cost),
\label{eq:conditional-family}
\end{equation}
where \(t\) is the absolute time, \(h=H-t\) is the remaining horizon, and the parameters \(\theta\) are shared over the complete family. Section~4 specifies how this function is learned and converted into a stopping policy.

\section{Method}

\subsection{\method{} overview}

\method{} receives a causal decision state \(S_t\) and an operating condition \((\cost,H)\), and returns a finite-horizon stopping policy. The solver has three connected components: a continuation model shared over cost and remaining horizon, a joint amortized fitted backward-induction method, and a structure-preserving cost head. Figure~\ref{fig:cc_aos_overview} separates the training recursion from deployment with the resulting checkpoint.

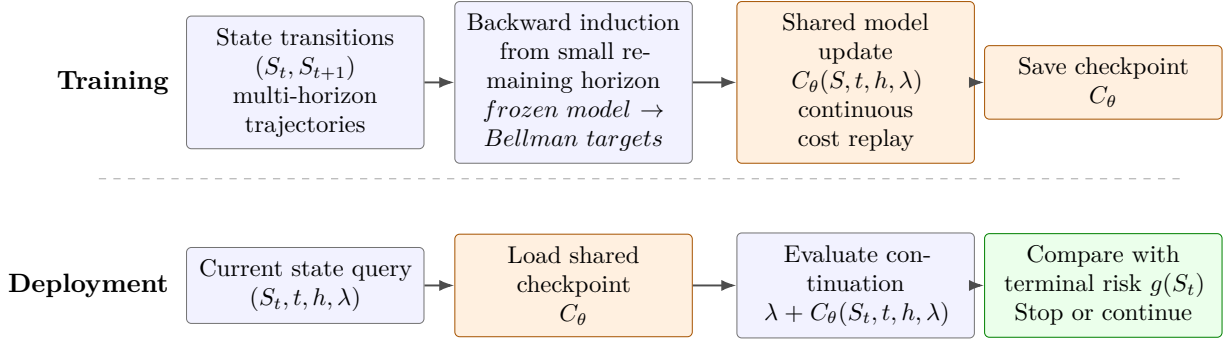
\begin{figure}[H]
\centering
\resizebox{0.98\linewidth}{!}{%
\begin{tikzpicture}[
  font=\small,
  box/.style={
    draw=black!55,
    rounded corners=2pt,
    fill=blue!5,
    align=center,
    minimum height=1.0cm,
    text width=3.1cm,
    inner sep=5pt
  },
  key/.style={
    box,
    fill=orange!12,
    draw=orange!65!black
  },
  decision/.style={
    box,
    fill=green!8,
    draw=green!50!black
  },
  arrow/.style={
    -{Latex[length=2.2mm]},
    line width=0.8pt,
    draw=black!70
  }
]

% ================= Training =================

\node[
font=\bfseries,
anchor=east
] at (0.55,1.6) {Training};

\node[box] (trajectory) at (2.4,1.6)
{State transitions\\
$(S_t,S_{t+1})$\\
multi-horizon trajectories};

\node[box] (horizon) at (6.3,1.6)
{Backward induction\\
from small remaining horizon\\
$frozen\ model\rightarrow Bellman\ targets$};

\node[key] (update) at (10.4,1.6)
{Shared model update\\
$C_\theta(S,t,h,\cost)$\\
continuous cost replay};

\node[key] (checkpoint) at (14.0,1.6)
{Save checkpoint\\
$C_\theta$};

\draw[arrow] (trajectory) -- (horizon);
\draw[arrow] (horizon) -- (update);
\draw[arrow] (update) -- (checkpoint);

% ================= separator =================

\draw[dashed, black!35]
(-0.6,0.2) -- (15.2,0.2);

% ================= Deployment =================

\node[
font=\bfseries,
anchor=east
] at (0.55,-1.35) {Deployment};

\node[box] (state)
at (2.4,-1.35)
{Current state query\\
$(S_t,t,h,\cost)$};

\node[key] (load)
at (6.3,-1.35)
{Load shared\\
checkpoint\\
$C_\theta$};

\node[box] (evaluate)
at (10.4,-1.35)
{Evaluate continuation\\
$\cost+C_\theta(S_t,t,h,\cost)$};

\node[decision] (stop)
at (14.0,-1.35)
{Compare with\\
terminal risk \(g(S_t)\)\\
Stop or continue};

\draw[arrow] (state) -- (load);
\draw[arrow] (load) -- (evaluate);
\draw[arrow] (evaluate) -- (stop);

\end{tikzpicture}%
}

\caption{
CC-AOS training and deployment framework.
During training, multi-horizon trajectories are used to generate recursive Bellman targets from frozen models.
A shared structured continuation model $C_\theta$ is updated through continuous-cost replay and stored as a checkpoint.
During deployment, the checkpoint is queried with the current operating condition to determine the stop-or-continue decision.
}

\label{fig:cc_aos_overview}

\end{figure}

\subsection{Shared cost--horizon continuation model}

\method{} parameterizes the continuation term in \eqref{eq:bellman} by \(C_\theta(s,t,h,\cost)\). The corresponding approximate value is
\begin{align}
\widehat V_0(s;\cost)
&=g(s),
\label{eq:vhat-zero}\\
\widehat V_h(s,t;\cost)
&=
\min\!\left\{
g(s),\,\cost+C_\theta(s,t,h,\cost)
\right\},
\qquad h\ge1.
\label{eq:vhat}
\end{align}
A shared feature network maps \((s,t,h)\) to the parameters of the cost head described in Section~\ref{sec:structured-head}. The state \(s\) represents the observed prefix, \(t\) captures absolute-time effects, and \(h=H-t\) specifies the depth of the remaining Bellman recursion. Because \(t+h=H\), the pair \((t,h)\) identifies both the current decision stage and its total finite horizon.

\subsection{Joint amortized fitted backward induction}

The identity in \eqref{eq:exact-continuation} turns continuation learning into conditional regression. Let \((s_t,s_{t+1})\) be a sampled transition, let \(h\ge1\), and draw \(\cost\) from a training distribution \(q\) on \(\Costs\). Given a frozen parameter snapshot \(\bar\theta\), define the recursive target
\begin{equation}
y_h=
\begin{cases}
g(s_{t+1}), & h=1,\\[2mm]
\min\!\left\{
g(s_{t+1}),
\cost+C_{\bar\theta}(s_{t+1},t+1,h-1,\cost)
\right\}, & h>1.
\end{cases}
\label{eq:sample-target}
\end{equation}
Conditioning on \(S_t=s\) and averaging over successor states makes \(y_h\) an estimate of \(P\widehat V_{h-1}\). The shared model is fitted by
\begin{equation}
\mathcal{L}(\theta)
=
\E_{\substack{(S_t,S_{t+1})\sim\mathcal{D}\\
h\sim\mathcal{H}_{\mathrm{train}},\;\cost\sim q}}
\!\left[
\ell_{\mathrm{reg}}\!\left(
C_\theta(s_t,t,h,\cost),y_h
\right)
\right].
\label{eq:training-loss}
\end{equation}
Figure~\ref{fig:cc_aos_training} details the resulting layerwise training process. Its inputs are the state trajectories \(\mathcal{D}\), the set of training horizons \(\Htrain\), the cost interval \(\Costs\), and the terminal risk \(g\). Starting from \(h=1\), the method initializes a shared continuation model and an empty cumulative replay buffer. Before targets are constructed at each remaining-horizon layer, the current parameters are frozen as \(\bar\theta\).

For every training horizon compatible with the current \(h\), the corresponding transition time is \(t=H-h\). Transitions \((s_t,s_{t+1})\) are paired with stratified cost samples \(\cost\sim q\) on \(\Costs\), and the frozen model generates the detached target in \eqref{eq:sample-target}. The resulting tuples \((s_t,t,h,\cost,y_h)\) are added to the replay buffer. The current parameters are then updated on all layers accumulated so far using \eqref{eq:training-loss}, after which \(h\) is incremented and the same cycle is repeated. The final output is one continuation-model checkpoint shared across continuous costs, total horizons, and Bellman layers.

\begin{figure}[t]
\centering
\includegraphics[width=\linewidth]{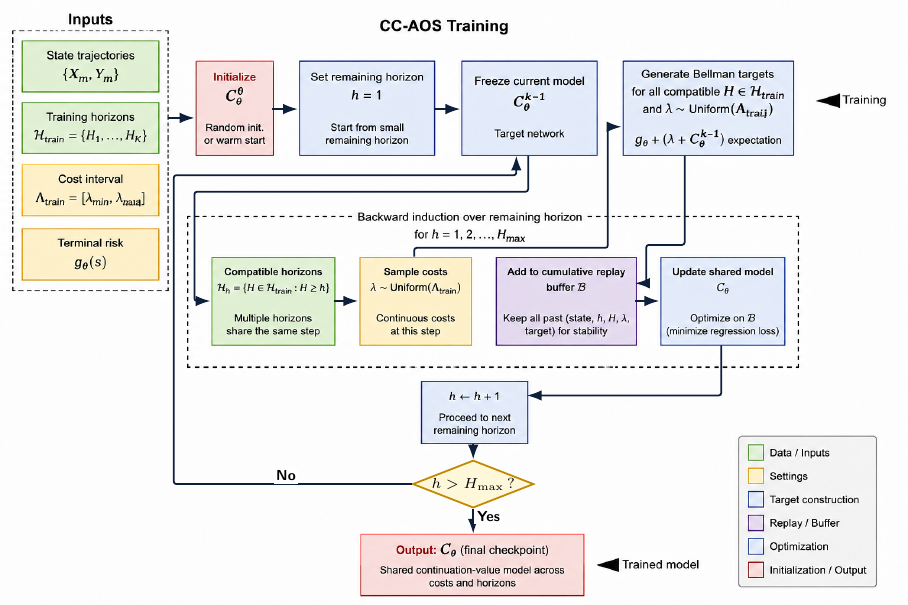}
\caption{Detailed training workflow of \method{}. Backward induction proceeds from small to large remaining horizons. At each layer, a frozen model generates Bellman regression targets for all compatible training horizons and stratified costs sampled over \(\Costs\); the resulting tuples are accumulated in a replay buffer and used to update the shared continuation model before advancing to the next layer. A single cost- and horizon-conditioned checkpoint is retained after the final layer.}
\label{fig:cc_aos_training}
\end{figure}

\subsection{Structure-preserving continuation head}
\label{sec:structured-head}

For fixed \((s,t,h)\), the feature network outputs \(J\) intercept logits and \(J\) slope logits. They are transformed as
\begin{align}
a_j(s,t,h)
&=B\,\sigma\!\left(\widetilde a_j(s,t,h)\right),
\label{eq:intercepts}\\
b_j(s,t,h)
&=(h-1)\,\sigma\!\left(\widetilde b_j(s,t,h)\right),
\label{eq:slopes}
\end{align}
so \(0\le a_j\le B\) and \(0\le b_j\le h-1\). The continuation head is the normalized soft minimum
\begin{equation}
\widetilde C_\theta(s,t,h,\cost)
=
-\rho\log\left[
\frac{1}{J}
\sum_{j=1}^J
\exp\left(
-\frac{a_j(s,t,h)+b_j(s,t,h)\cost}{\rho}
\right)
\right],
\label{eq:softmin}
\end{equation}
where \(\rho>0\) is the temperature. The factor \(1/J\) ensures that identical affine components return their common value. The final estimate is capped at the terminal-risk upper bound,
\begin{equation}
C_\theta(s,t,h,\cost)
=
\min\{B,\widetilde C_\theta(s,t,h,\cost)\}.
\label{eq:continuation-cap}
\end{equation}
For binary zero--one loss, \(B=0.5\). The slope limit \(h-1\) matches the maximum number of acquisitions remaining after the current observation is obtained. The explicit cost in \eqref{eq:vhat} supplies the final unit of cost sensitivity. Section~5 proves the resulting monotonicity, concavity, and Lipschitz properties.

\subsection{Stopping policy and amortized deployment}

Given a query \((s,t,h,\cost)\), \method{} compares the stopping and continuation risks:
\begin{equation}
\widehat\pi_h(s,t;\cost)
=
\begin{cases}
\text{stop}, & g(s)\le \cost+C_\theta(s,t,h,\cost),\\
\text{continue}, & \text{otherwise}.
\end{cases}
\label{eq:greedy-policy}
\end{equation}
If the policy continues, one observation is acquired and \((t,h)\) becomes \((t+1,h-1)\); stopping is mandatory at \(h=0\). Because \(\cost\) and \(h\) are query variables, the same checkpoint constructs stopping policies for all operating conditions represented by the learned continuation family.

\section{Theoretical Analysis}

This section first characterizes the cost geometry of the exact finite-horizon problem, then proves that the structured head preserves it. We finally connect fitted Bellman errors to value and policy performance. Throughout, the horizon is finite, \(g(S_\tau)\) is integrable, and the state transition law is independent of \(\cost\).

\subsection{Cost geometry of exact values and continuations}

\begin{theorem}[Exact value geometry]
\label{thm:exact-geometry}
For every state \(s\) and finite remaining horizon \(h\), \(V_h(s;\cost)\) is nondecreasing and concave in \(\cost\ge0\). Moreover, for all \(\cost_1,\cost_2\ge0\),
\begin{equation}
\left|V_h(s;\cost_2)-V_h(s;\cost_1)\right|
\le
h|\cost_2-\cost_1|.
\label{eq:value-lipschitz}
\end{equation}
\end{theorem}

\begin{proof}
For an admissible stopping time \(\tau\le h\), define
\[
L_\tau(\cost)
=
\E_s[g(S_\tau)]
+\cost\,\E_s[\tau].
\]
This is affine in \(\cost\) with slope \(\E_s[\tau]\in[0,h]\). Equation~\eqref{eq:stopping-objective} is the pointwise infimum of the family \(\{L_\tau\}\), hence is concave. Since every member is nondecreasing, the infimum is nondecreasing. For \(\cost_2\ge\cost_1\), evaluating an \(\varepsilon\)-optimal rule for \(\cost_1\) at \(\cost_2\) gives
\[
V_h(s;\cost_2)-V_h(s;\cost_1)
\le h(\cost_2-\cost_1)+\varepsilon.
\]
Letting \(\varepsilon\downarrow0\), and using nondecreasingness for the other direction, proves \eqref{eq:value-lipschitz}.
\end{proof}

\begin{corollary}[Exact continuation geometry]
\label{cor:continuation-geometry}
For \(h\ge1\), \(C_h(s;\cost)=PV_{h-1}(s;\cost)\) is nondecreasing, concave, and \((h-1)\)-Lipschitz in \(\cost\). Consequently, the complete continuation risk \(\cost+C_h(s;\cost)\) has cost slopes in \([1,h]\) wherever it is differentiable.
\end{corollary}

\begin{proof}
Conditional expectation preserves nondecreasingness and concavity. Applying \eqref{eq:value-lipschitz} to \(V_{h-1}\) and then taking conditional expectation gives the \((h-1)\)-Lipschitz bound. Adding \(\cost\) shifts every slope by one.
\end{proof}

\begin{remark}[Envelope interpretation]
If an optimal rule exists and \(V_h(s;\cdot)\) is differentiable at \(\cost\), then
\[
\partial_\cost V_h(s;\cost)=\E_s[\tau^\star].
\]
At a nondifferentiable cost, the superdifferential is the convex hull of the expected stopping times of active optimal rules. This statement concerns the exact value and motivates the admissible slope range; it is not an identity imposed on finite differences of a learned approximation.
\end{remark}

\subsection{Structure preservation by the conditional head}

\begin{theorem}[Structure-preserving approximation]
\label{thm:head}
Assume \(\cost\ge0\), \(0\le a_j\le B\), and \(0\le b_j\le h-1\) for every component in \eqref{eq:softmin}. For every \(h\ge1\), \(C_\theta(s,t,h,\cdot)\) is nondecreasing, concave, and \((h-1)\)-Lipschitz, with \(0\le C_\theta\le B\). The induced value \(\widehat V_h(s,t;\cdot)\) in \eqref{eq:vhat} is nondecreasing, concave, and \(h\)-Lipschitz. The terminal boundary \(\widehat V_0=g\) holds exactly.
\end{theorem}

\begin{proof}
Let \(z_j(\cost)=a_j+b_j\cost\). Differentiating the uncapped head yields
\begin{equation}
\partial_\cost\widetilde C_\theta
=
\sum_{j=1}^J w_j(\cost)b_j,
\qquad
w_j(\cost)
=
\frac{\exp[-z_j(\cost)/\rho]}
{\sum_k\exp[-z_k(\cost)/\rho]}.
\label{eq:softmin-derivative}
\end{equation}
The derivative is a convex combination of slopes in \([0,h-1]\). The negative log-sum-exp of negative affine functions is concave, while the factor \(1/J\) adds only a constant. Because every \(z_j\ge0\), the normalized soft minimum is nonnegative. Taking the minimum with the constant \(B\) preserves nondecreasingness and concavity and gives \(0\le C_\theta\le B\). Thus \(C_\theta\) is \((h-1)\)-Lipschitz.

The function \(\cost+C_\theta\) is nondecreasing and concave with slope bounded by \(h\). Taking its pointwise minimum with the constant \(g(s)\) preserves these properties. The terminal statement follows from \eqref{eq:vhat-zero}.
\end{proof}

\subsection{Uniform Bellman residual-to-value bound}

Define the Bellman operator
\begin{equation}
(T_\cost f)(s)
=
\min\{g(s),\cost+(Pf)(s)\}.
\label{eq:bellman-operator}
\end{equation}
For bounded functions \(f\) and \(q\),
\begin{equation}
\norm{T_\cost f-T_\cost q}_\infty
\le
\norm{f-q}_\infty,
\label{eq:nonexpansive}
\end{equation}
because both \(P\) and the scalar minimum are non-expansive in the supremum norm. Define the layerwise population Bellman residual
\begin{equation}
\epsilon_h(\cost)
=
\norm{\widehat V_h-T_\cost\widehat V_{h-1}}_\infty.
\label{eq:bellman-residual}
\end{equation}

\begin{theorem}[Residual-to-value bound]
\label{thm:residual}
If \(\widehat V_0=V_0=g\), then for every \(h\ge1\),
\begin{equation}
\norm{\widehat V_h-V_h}_\infty
\le
\sum_{k=1}^h\epsilon_k(\cost).
\label{eq:value-error}
\end{equation}
If \(\epsilon_k(\cost)\le\bar\epsilon_k\) uniformly for \(\cost\in\Costs\), then
\begin{equation}
\sup_{\cost\in\Costs}
\norm{\widehat V_h(\cdot;\cost)-V_h(\cdot;\cost)}_\infty
\le
\sum_{k=1}^h\bar\epsilon_k.
\label{eq:uniform-value-error}
\end{equation}
\end{theorem}

\begin{proof}
Adding and subtracting \(T_\cost\widehat V_{h-1}\), then applying \eqref{eq:nonexpansive}, gives
\[
\norm{\widehat V_h-V_h}_\infty
\le
\epsilon_h(\cost)
+
\norm{\widehat V_{h-1}-V_{h-1}}_\infty.
\]
Iterating from the exact terminal boundary proves \eqref{eq:value-error}; taking the supremum over \(\Costs\) proves \eqref{eq:uniform-value-error}.
\end{proof}

\subsection{Greedy-policy performance from continuation residuals}

The deployed policy uses \(C_{\theta,h}\) directly. Define its continuation residual by
\begin{equation}
\delta_h(\cost)
=
\norm{C_{\theta,h}-P\widehat V_{h-1}}_\infty.
\label{eq:continuation-residual}
\end{equation}
The minimum operator is 1-Lipschitz, so \(\epsilon_h(\cost)\le\delta_h(\cost)\). Let \(V_h^{\widehat\pi}\) be the value obtained by the greedy policy in \eqref{eq:greedy-policy}.

\begin{theorem}[Greedy-policy performance]
\label{thm:policy}
For every \(s\), \(h\), and \(\cost\),
\begin{equation}
0
\le
V_h^{\widehat\pi}(s;\cost)-V_h(s;\cost)
\le
2\sum_{k=1}^h\delta_k(\cost).
\label{eq:policy-bound}
\end{equation}
The same statement holds uniformly over \(\Costs\) when the continuation residuals admit uniform bounds on that interval.
\end{theorem}

\begin{proof}
Let \(D_h=\norm{V_h^{\widehat\pi}-\widehat V_h}_\infty\). If the greedy rule stops, both quantities equal \(g\). If it continues, the policy recursion and \eqref{eq:continuation-residual} give
\[
D_h
\le
D_{h-1}+\delta_h(\cost).
\]
Since \(D_0=0\), \(D_h\le\sum_{k=1}^h\delta_k(\cost)\). Theorem~\ref{thm:residual} and \(\epsilon_k\le\delta_k\) also give
\[
\norm{\widehat V_h-V_h}_\infty
\le
\sum_{k=1}^h\delta_k(\cost).
\]
Finally, \(V_h\le V_h^{\widehat\pi}\) because \(V_h\) is optimal. Combining the last two inequalities with the triangle inequality proves \eqref{eq:policy-bound}.
\end{proof}

\subsection{Theoretical implications}

The results establish four linked properties. First, the structured head is reliable along the continuous cost axis by construction: its continuation values and induced full values satisfy the exact monotonicity, concavity, and remaining-horizon Lipschitz pattern. Second, the terminal boundary is exact and the continuation slope range matches the number of future acquisitions still available. Third, uniformly small Bellman residuals control the value approximation accumulated over a finite horizon. Fourth, uniformly small continuation residuals connect fitted continuation learning to an explicit performance bound for the deployed greedy policy. The cost-geometry audit, held-out Bellman diagnostics, and direct policy evaluation in Section~\ref{sec:experiments} are organized around these quantities.

\section{Experiments}
\label{sec:experiments}

\subsection{Experimental setup}

\paragraph{Evaluation settings.}
Three complementary settings are considered. The Gaussian/DRE setting is a controlled binary sequential problem in which a frozen density-ratio estimator supplies the log-likelihood-ratio state used by every stopping solver. It evaluates transfer to six held-out cost--horizon pairs under an estimated state representation. The time-varying non-Gaussian mixture changes its weights, means, and variances over time and uses an oracle log-likelihood ratio, thereby separating stopping approximation from state-estimation error. It evaluates held-out operating points and includes one independently fitted CFL anchor.

FordA is the primary real time-series benchmark and uses its standard train/test split. Its frozen posterior model contains only causal convolutions and per-time-step LayerNorm, so every stopping decision uses the observed prefix alone. All stopping methods reuse the same leakage-free state trajectories. Test prefix accuracy rises from 58.71\% at the first block to 76.36\%, 84.24\%, 90.00\%, 94.09\%, and 95.61\% at blocks 5, 10, 20, 35, and 50, respectively, yielding a nontrivial accuracy--earliness trade-off.

\paragraph{Compared methods.}
We compare \method{} with three references. Per-setting FIRMBOUND with Convex Function Learning (CFL) fits an independent finite-horizon solver at each requested operating point. The per-horizon reference retains the continuation architecture and continuous-cost training used by \method{} but fits a separate model for each horizon, isolating the effect of horizon amortization. The static reference selects a stopping threshold on validation data separately for every \((\cost,H)\). Within each setting, all methods use the same state trajectories and terminal-risk definition.

\paragraph{Metrics and statistics.}
Let \(T_i\) be the stopping time selected for test trajectory \(i\), and let \(\tau_i=T_i-t_0\) be the number of samples acquired after the initial decision time \(t_0\). The primary metric is the empirical posterior objective
\begin{equation}
\widehat J(\cost,H)
=
\frac{1}{N}\sum_{i=1}^{N}
\left[
g\!\left(S_{T_i}^{(i)}\right)+\cost\tau_i
\right],
\label{eq:empirical-posterior-objective}
\end{equation}
which combines terminal posterior risk and sampling cost; lower values are better. Classification error and the expected number of additional samples are retained as diagnostic quantities. Main pairwise comparisons use paired bootstrap confidence intervals over the common FordA test trajectories. The reported FordA checkpoint comes from one frozen training seed, so these intervals quantify test-trajectory variation rather than variation across training runs.

\subsection{Generalization to unseen operating points}

\paragraph{Primary FordA comparison.}
One \method{} checkpoint is trained on \(H=\{20,30,40,50\}\) and \(\cost\in[0.0005,0.008]\), then evaluated directly at the six interleaved pairs formed by \(H=\{25,35,45\}\) and \(\cost=\{0.002375,0.006125\}\). Table~\ref{tab:forda-main} compares the posterior objectives, and Table~\ref{tab:forda-paired} reports the paired trajectory-level differences from CFL. \method{} has a lower posterior objective than per-setting CFL at all six pairs. Relative reductions range from 0.95\% to 31.29\%, with a mean of 15.75\%. The paired-bootstrap 95\% interval lies entirely below zero at five pairs; the interval at \(H=25,\cost=0.006125\) crosses zero slightly.

\begin{table}[H]
\centering
\begin{threeparttable}
\caption{Posterior objective of \method{} and baselines at six unseen FordA operating points.}
\label{tab:forda-main}
\footnotesize
\begin{tabular*}{\linewidth}{@{\extracolsep{\fill}}ccccc@{}}
\toprule
& & \multicolumn{3}{c}{Posterior objective \(\downarrow\)}\\
\cmidrule(lr){3-5}
\(H\) & \(\cost\) & \method{} & CFL & Static\\
\midrule
25 & 0.002375 & \textbf{0.11565} & 0.12472 & \underline{0.11659}\\
25 & 0.006125 & \textbf{0.16266} & \underline{0.16423} & 0.16445\\
35 & 0.002375 & \textbf{0.10129} & 0.13041 & \underline{0.10268}\\
35 & 0.006125 & \underline{0.15625} & 0.19220 & \textbf{0.15311}\\
45 & 0.002375 & \underline{0.09432} & 0.13726 & \textbf{0.09126}\\
45 & 0.006125 & \underline{0.15309} & 0.17798 & \textbf{0.14839}\\
\midrule
\multicolumn{2}{@{}l}{Six-pair mean}
& \underline{0.13054} & 0.15447 & \textbf{0.12941}\\
\bottomrule
\end{tabular*}
\begin{tablenotes}[flushleft]
\footnotesize
\item Lower is better. Bold and underline mark the best and second-best objective in each row, respectively. All values use the causal per-time-step LayerNorm pipeline.
\end{tablenotes}
\end{threeparttable}
\end{table}

\begin{table}[H]
\centering
\begin{threeparttable}
\caption{Paired trajectory-level comparison between \method{} and per-setting CFL.}
\label{tab:forda-paired}
\footnotesize
\begin{tabular*}{\linewidth}{@{\extracolsep{\fill}}ccccc@{}}
\toprule
\(H\) & \(\cost\) & \(\Delta\downarrow\) & 95\% CI & Relative reduction (\%) \(\uparrow\)\\
\midrule
25 & 0.002375 & -0.00907 & [\(-0.01148,-0.00646\)] & 7.27\\
25 & 0.006125 & -0.00156 & [\(-0.00334,\phantom{-}0.00018\)] & 0.95\\
35 & 0.002375 & -0.02912 & [\(-0.03162,-0.02669\)] & 22.33\\
35 & 0.006125 & -0.03595 & [\(-0.03967,-0.03225\)] & 18.70\\
45 & 0.002375 & -0.04295 & [\(-0.04596,-0.04007\)] & 31.29\\
45 & 0.006125 & -0.02489 & [\(-0.02891,-0.02114\)] & 13.98\\
\midrule
\multicolumn{2}{@{}l}{Six-pair mean} & -0.02392 & --- & 15.75\\
\bottomrule
\end{tabular*}
\begin{tablenotes}[flushleft]
\footnotesize
\item \(\Delta=J_{\mathrm{CC\text{-}AOS}}-J_{\mathrm{CFL}}\); negative values favor \method{}. Confidence intervals are paired bootstrap intervals over the common FordA test trajectories.
\end{tablenotes}
\end{threeparttable}
\end{table}

Against the same-architecture per-horizon reference, the mean positive posterior-objective gap is 0.49\% and the maximum is 1.66\%, indicating a small empirical cost for sharing across horizons. Against the static threshold, \method{} wins at three of six pairs and has a six-pair mean of 0.13054 versus 0.12941. The two methods are therefore essentially tied on the six-pair average.

\paragraph{Controlled-process validation.}
The controlled settings test whether the same conclusion persists under different state and transition mechanisms. In the Gaussian/DRE setting, the shared model improves on per-setting CFL at all six held-out pairs, with an average objective reduction of 10.99\%. In the time-varying non-Gaussian mixture, the mean positive gap to same-architecture per-horizon models is 0.55\% and the maximum is 1.53\%. At the anchor \(H=35,\cost=0.00325\), \method{} obtains \textbf{0.25041}, compared with \underline{0.28250} for CFL, an 11.36\% reduction.

\paragraph{Bellman-fit diagnostic.}
As a separate diagnostic of the joint fitted recursion, we evaluate held-out one-step full-value residuals using the target definition in Section~4.3 on FordA states excluded from fitting. Across the audited operating points, the residual MAE ranges from \(0.00263\) to \(0.01102\), the RMSE from \(0.00975\) to \(0.02313\), and the 95th percentile from \(0.00801\) to \(0.04506\). These finite-sample quantities provide empirical counterparts to the layerwise approximation errors in Theorem~\ref{thm:residual} and diagnose how closely the learned recursion matches its held-out Bellman targets.

\subsection{Ablation of the structure-preserving head}

To isolate the contribution of the cost-structure constraints, we compare the normalized soft-min affine head with a parameter-matched unstructured head on the controlled DRE task. The two variants use the same state representation, recursive targets, and training protocol; only the continuation head is changed. Table~\ref{tab:head-ablation} reports both task performance and violations of the cost-dependent properties established in Section~5.

\begin{table}[H]
\centering
\begin{threeparttable}
\caption{Parameter-matched ablation of the structure-preserving continuation head on the controlled DRE task.}
\label{tab:head-ablation}
\small
\renewcommand{\arraystretch}{1.12}
\begin{tabularx}{\linewidth}{@{}p{0.40\linewidth}>{\centering\arraybackslash}X>{\centering\arraybackslash}X@{}}
\toprule
Metric & Structured head & {\makecell[c]{Parameter-matched\\unstructured head}}\\
\midrule
Mean task objective \(\downarrow\)
& \underline{0.059501}
& \textbf{0.059223}\\
\addlinespace
Monotonicity violations \(\downarrow\)
& \textbf{0}
& \underline{7.10\%}\\
\addlinespace
Continuation-concavity violations \(\downarrow\)
& \textbf{0}
& \underline{1.91\%}\\
\addlinespace
Value-concavity violations \(\downarrow\)
& \textbf{0}
& \underline{1.69\%}\\
\addlinespace
Value-Lipschitz violations \(\downarrow\)
& \textbf{0}
& \underline{0.38\%}\\
\bottomrule
\end{tabularx}
\begin{tablenotes}[flushleft]
\footnotesize
\item Lower is better for every metric. Bold and underline mark the best and second-best values, respectively. Violation rates are proportions of audited cost cells.
\end{tablenotes}
\end{threeparttable}
\end{table}

The unstructured head has the lower mean task objective by \(0.000278\), whereas the structured head eliminates all audited monotonicity, continuation-concavity, value-concavity, and value-Lipschitz violations. The parameter-matched comparison therefore attributes the structural behavior to the proposed head rather than to model size, with only a small observed difference in task objective. A complementary FordA audit over 21 costs, three unseen horizons, and 256 trajectories examines 2,962,080 slope cells and detects no violation by the structured model.

\subsection{Computational efficiency}

Table~\ref{tab:efficiency} reports the observed training cost and number of independently fitted solver instances for the evaluated FordA configurations. \method{} trains one shared model in 18.04 s. Training three same-architecture per-horizon references takes 33.48 s in total, while fitting CFL independently at the six evaluated operating points takes approximately 53 min.

\begin{table}[H]
\centering
\begin{threeparttable}
\caption{Observed training and model-management cost on FordA.}
\label{tab:efficiency}
\small
\renewcommand{\arraystretch}{1.12}
\begin{tabularx}{\linewidth}{@{}X X>{\centering\arraybackslash}p{0.18\linewidth}>{\centering\arraybackslash}p{0.20\linewidth}@{}}
\toprule
Method & Training scope & Independent fits \(\downarrow\) & {\makecell[c]{Total wall-clock\\time \(\downarrow\)}}\\
\midrule
\method{}
& Continuous costs and multiple horizons
& \textbf{1}
& \textbf{18.04 s}\\
\addlinespace
Per-horizon reference
& Continuous costs; one model per horizon
& \underline{3}
& \underline{33.48 s}\\
\addlinespace
Per-setting CFL
& One solver per \((\cost,H)\)
& 6
& \(\approx 53\) min\\
\bottomrule
\end{tabularx}
\begin{tablenotes}[flushleft]
\footnotesize
\item Lower is better. Bold and underline mark the best and second-best values, respectively. Wall-clock values are observed measurements for the evaluated configurations.
\end{tablenotes}
\end{threeparttable}
\end{table}

The timing comparison follows the deployment granularity of the methods: \method{} produces one checkpoint for the complete queried family, the per-horizon reference produces three checkpoints, and CFL produces six operating-point-specific solver instances. This observed comparison directly measures the amortization achieved in the evaluated setting.

\section{Discussion}

\method{} converts a collection of operating-point-specific finite-horizon solvers into one cost- and horizon-conditioned fitted stopping model. Across the leakage-free FordA evaluation, one checkpoint serves unseen cost--horizon combinations and improves on per-setting CFL at every tested pair. The controlled Gaussian/DRE and time-varying mixture results provide complementary evidence under estimated and oracle state representations.

The practical value is not that \method{} produces the lowest objective under every comparison. The strong static threshold is competitive and is marginally better on the six-pair FordA average. Rather, \method{} supplies a dynamic, state-dependent stopping boundary, a shared continuation surface, and direct multi-operating-point deployment. A static threshold does not represent a general finite-horizon Bellman boundary or expose a continuation value that can be queried across cost and remaining horizon.

\method{} retains the data-driven state interface used by FIRMBOUND while changing the granularity of the stopping solver. Instead of fitting a continuation stack for one requested operating point, it fits a conditional family through shared recursive targets. The structure-preserving head gives this family verifiable behavior along the continuous cost axis. The observed training times and one-checkpoint representation further support its use when several latency or acquisition-cost configurations must be maintained.

\subsection{Limitations and future work}

The current real-data evidence is concentrated on binary FordA, and the primary checkpoint comes from one frozen training seed; paired bootstrap intervals quantify variation across common test trajectories rather than training randomness. The held-out operating points mainly test cost interpolation within the modeled interval and interleaved-horizon generalization. In addition, the stopping solver depends on a fixed upstream causal state estimator, and the present formulation assumes a constant per-observation cost. Natural extensions include multiclass or higher-dimensional belief states, larger real sequential benchmarks, online cost calibration, time- or state-dependent acquisition costs, and verifiable structure along the horizon axis.

\section{Conclusion}

Finite-horizon optimal stopping is commonly re-solved when the sampling cost or horizon changes, limiting efficient multi-operating-point deployment. \method{} represents a family of such problems with one state-, time-, horizon-, and cost-conditioned continuation model learned by joint amortized fitted backward induction. Its normalized soft-min affine head translates exact monotonicity, concavity, and remaining-horizon Lipschitz geometry into by-construction guarantees, while Bellman and continuation residuals connect approximation quality to value and greedy-policy performance. Across controlled Gaussian/DRE, time-varying non-Gaussian, and leakage-free FordA experiments, one checkpoint serves unseen cost--horizon combinations; on FordA it improves on per-setting CFL at all six pairs and remains essentially tied with a strong static threshold on average. The results support structured amortization as a practical route from repeated finite-horizon solver stacks to one geometrically well-behaved model that can be queried directly across operating points.

\bibliographystyle{unsrtnat}
\bibliography{references}

\end{document}